\documentclass[letterpaper]{article} 
\usepackage{aaai23}  
\usepackage{times}  
\usepackage{helvet}  
\usepackage{courier}  
\usepackage[hyphens]{url}  
\usepackage{graphicx} 
\usepackage{natbib}  
\usepackage{caption} 
\usepackage{algorithm}
\usepackage{algorithmic}

\usepackage{xcolor}
\definecolor{Klein_Blue}{rgb}{0.1215,0.0941,0.752}

\usepackage{hyperref}

\hypersetup{
colorlinks=true,
linkcolor=Klein_Blue,
filecolor=Klein_Blue,
citecolor = Klein_Blue,
urlcolor=Klein_Blue,
}

\usepackage{newfloat}
\usepackage{listings}
\DeclareCaptionStyle{ruled}{labelfont=normalfont,labelsep=colon,strut=off} 
\floatstyle{ruled}
\newfloat{listing}{tb}{lst}{}
\floatname{listing}{Listing}
\title{On Robust Learning from Noisy Labels: A Permutation Layer Approach}
\author{
    Salman Alsubaihi\textsuperscript{\rm 1}\equalcontrib,
    Mohammed Alkhrashi\textsuperscript{\rm 1}\equalcontrib,
    Raied Aljadaany\textsuperscript{\rm 1},
    Fahad Albalawi\textsuperscript{\rm 1},
    Bernard Ghanem\textsuperscript{\rm 2},
}
\affiliations{
    \textsuperscript{\rm 1}Saudi Data and Artificial Intelligence Authority (SDAIA), Riyadh, Saudi Arabia\\
    \textsuperscript{\rm 2}King Abdullah University of Science and Technology (KAUST), Thuwal, Saudi Arabia\\
    \{salsubaihi, mkhrashi, raljadaany, falbalawi\}@sdaia.gov.sa, bernard.ghanem@kaust.edu.sa

}

\usepackage{bibentry}

\usepackage[utf8]{inputenc}
\usepackage{pgfplots}
\DeclareUnicodeCharacter{2212}{−}
\usepgfplotslibrary{groupplots,dateplot}
\usetikzlibrary{patterns,shapes.arrows}
\pgfplotsset{compat=newest}

\usepackage{amsfonts}       
\usepackage{xcolor}         
\usepackage{amsmath}
\usepackage{amsthm}
\usepackage{xfrac}
\usepackage{relsize}

\newcommand{\permmatrix}[1]{\xoverline{P}_{\!\!\boldsymbol{\alpha}^{#1}}}

\newcommand{\alphaInit}{I_{\alpha}}
\newcommand{\myPermLR}{\eta_{\alpha}}

\newtheorem{thm}{Theorem}
\newtheorem{prop}[thm]{Proposition}

\newtheorem{definition}{Definition}
\newtheorem{assumption}{Assumption}

\usepackage{calrsfs}
\DeclareMathAlphabet{\pazocal}{OMS}{zplm}{m}{n}
\renewenvironment{proof}{\noindent\textbf{Proof.}}{\qed}

\makeatletter
\newsavebox\myboxA
\newsavebox\myboxB
\newlength\mylenA
\newcommand*\xoverline[2][0.75]{%
    \sbox{\myboxA}{$\m@th#2$}%
    \setbox\myboxB\null
    \ht\myboxB=\ht\myboxA%
    \dp\myboxB=\dp\myboxA%
    \wd\myboxB=#1\wd\myboxA
    \sbox\myboxB{$\m@th\overline{\copy\myboxB}$}
    \setlength\mylenA{\the\wd\myboxA}
    \addtolength\mylenA{-\the\wd\myboxB}%
    \ifdim\wd\myboxB<\wd\myboxA%
       \rlap{\hskip 0.5\mylenA\usebox\myboxB}{\usebox\myboxA}%
    \else
        \hskip -0.5\mylenA\rlap{\usebox\myboxA}{\hskip 0.5\mylenA\usebox\myboxB}%
    \fi}
\makeatother

\usepackage{multirow} 
\usepackage{booktabs} 
\usepackage{textcomp} 
\usepackage{dblfloatfix}

\usepackage[utf8]{inputenc}
\usepackage{pgfplots}
\DeclareUnicodeCharacter{2212}{−}
\usepgfplotslibrary{groupplots,dateplot}
\usetikzlibrary{patterns,shapes.arrows}
\pgfplotsset{compat=newest}

\begin{document}

\maketitle

\begin{abstract}
The existence of label noise imposes significant challenges (e.g., poor generalization) on the training process of deep neural networks (DNN). As a remedy, this paper introduces a permutation layer learning approach termed PermLL to dynamically calibrate the training process of the DNN subject to instance-dependent and instance-independent label noise. The proposed method augments the architecture of a conventional DNN by an \emph{instance-dependent permutation layer}. This layer is essentially a convex combination of permutation matrices that is dynamically calibrated for each sample. The primary objective of the permutation layer is to correct the loss of noisy samples mitigating the effect of label noise. We provide two variants of PermLL in this paper: one applies the permutation layer to the model's prediction, while the other applies it directly to the given noisy label. In addition, we provide a theoretical comparison between the two variants and show that previous methods can be seen as one of the variants. Finally, we validate PermLL experimentally and show that it achieves state-of-the-art performance on both real and synthetic datasets.

\end{abstract}
\section{Introduction}
Deep Neural Networks (DNNs) have achieved outstanding performance on many vision problems, including image classification \cite{krizhevsky2012imagenet}, object detection \cite{redmon2016you}, semantic segmentation \cite{long2015fully}, and scene labeling \cite{farabet2012learning}. The success in these challenges heavily relies on the availability of huge and correctly labeled datasets, which are very expensive and time-demanding to collect. To overcome this challenge, a number of crowdsourcing platforms such as Amazon’s Mechanical Turk and nonexpert sources such as Internet Web Images, where labels are inferred by surrounding text or keywords, have been developed over the past years \cite{xiao2015learning}. Although these methods reduce the labeling cost, the labels derived from these techniques are unreliable due to the high noise rate caused by human annotators or extraction algorithms \cite{paolacci2010running,scott2013classification}. In addition, more challenging tasks usually require domain expert annotators and thus are highly prone to mislabeling \cite{frenay2013classification,lloyd2004observer}, such as breast tumor classification \cite{lee2018cleannet}. 

Training on such noisy datasets negatively impacts the performance of DDNs (i.e., poor generalization) due to the memorization effect \cite{maennel2020neural}. Interestingly, the authors in \cite{zhang2021understanding} illustrated that DNNs can easily fit randomly labeled training data.
This property is especially problematic when training in the presence of label noise since typical DNNs tend to memorize the noisy instances leading to a subpar classifier. 

To overcome this impediment, we propose a learnable permutation layer that is applied to one of the training loss arguments (i.e., model prediction or label). 
Each training sample has an independent permutation layer associated with a learnable parameter $\boldsymbol{\alpha}$.
The purpose of the permutation layer is to correct the loss of noisy samples through permuting predictions or labels during training, allowing the model to learn safely from them.
During inference, the permutation layer is discarded. Our contributions can be summarized as follows:
\begin{enumerate}
    \item We propose a permutation layer learning framework PermLL that can effectively learn from datasets containing noisy labels.
    \item We theoretically analyze two variants of PermLL. The first approach applies the permutation layer to the predictions, while the second applies the permutation to the labels. We show that the approach of learning the labels directly, proposed in Joint Optimization \cite{tanaka2018joint}, can be seen as a special case of PermLL.
    \item We provide a theoretical analysis of the two proposed methods, showing that applying the permutation layer to the predictions has better theoretical properties. 
    \item We empirically demonstrate the effectiveness of PermLL on synthetic and real noise, achieving state-of-the-art performance on CIFAR-10, CIFAR-100, and Clothing1M.
\end{enumerate}

\section{Related Work}
In recent years, a considerable number of methods have been proposed to deal with different types of label noise. Label noise can come from a closed-set or an open-set. 
In the closed-set case, the true label of a noisy sample is guaranteed to belong to the dataset's predefined set of labels, while in the open-set noise the true label could fall outside the predefined labels. 
Under closed-set noise, the label noise can be further categorized into \emph{instance-independent label noise} and \emph{instance-dependent label noise} \cite{natarajan2013learning,xiao2015learning}.
For instance-independent label noise, the noise is conditionally independent on the sample feature given its clean label. Therefore, the label noise can be characterized as a transition matrix \cite{frenay2013classification}.
Instance-dependent label noise, on the other hand, depends on both sample features and the true label \cite{xiao2015learning}. In this paper, we only consider a closed-set noise setting, covering both instance-dependent and instance-independent noise. 

Many methods have been proposed to learn robustly from noisy labeled datasets. These methods can be categorized based on how they handle corrupted samples into sample selection and loss correction methods. The reader may refer to \cite{song2022learning, frenay2013classification} for a more thorough review.

\paragraph{Sample selection.} Sample selection methods provide a heuristic to isolate clean samples from noisy ones. Then, the noisy samples are either discarded or used as unlabeled samples in a semi-supervised manner \cite{malach2017decoupling, han2018co, yu2019does, li2020dividemix, song2021robust}. To name a few, Decouple \cite{malach2017decoupling} trains two networks on samples where they have disagreements in the prediction. Co-teaching \cite{han2018co} uses two networks, each selecting small-loss samples in a mini-batch to train the other. Inspired by Decouple, Co-teaching+ \cite{yu2019does} improves on Co-teaching by training two networks on samples that have small-loss and prediction disagreement. Unlike these methods, our work follows a loss correction approach.



\paragraph{Loss correction} Methods of this type cope with label noise by correcting the loss of noisy samples. Therefore, this line of work is more relevant to our approach. Several methods modify the loss function by augmenting their model with a linear layer during training. Some works utilize an estimated transition matrix to correct the loss \cite{patrini2017making, hendrycks2018using, xia2019anchor, yao2020dual}. Notably, \cite{patrini2017making} first estimates the noise transition matrix using the model's predictions in an initial training phase. Then, the model is re-trained with either a \emph{forward} or \emph{backward} correction. The performance of such methods highly depends on the quality of the estimated transition matrix.
Similarly, other methods augment their models with a noise adaptation layer \cite{chen2015webly,sukhbaatar2014training}. Unlike transition matrix-based methods, the adaptation layer is learned simultaneously with the model. A major drawback of the previously mentioned methods, using a transition matrix or an adaptation layer, is that they treat all samples equally. PermLL, however, employs an instance-dependent permutation layer that can deal with samples differently.

Another prominent line of work mitigates the effect of label noise by refurbishing the labels. For example, Bootstrapping \cite{reed2014training} generates training targets that are obtained by a convex combination of the model prediction and the noisy labels using a single parameter $\beta$ for all samples. To remedy this, \cite{arazo2019unsupervised} dynamically adjusts the parameter $\beta$ for every sample using a beta mixture model fitted to the training loss of each sample. Such bootstrapping techniques require careful hyperparameter scheduling, increasing their complexity. Joint Optimization \cite{tanaka2018joint} proposes a more straightforward approach that jointly learns the model's parameters and the labels. In addition, \cite{tanaka2018joint} uses regularization terms to impose a known prior distribution on the labels and to avoid getting stuck in local optima.
Inspired by \cite{tanaka2018joint}, PENCIL \cite{yi2019probabilistic} learns label distributions instead of constant labels. 

In contrast to the aforementioned literature, our proposed method employs a learnable instance-dependent permutation layer that is applied to one of the training loss arguments (i.e., model prediction or label). From the previously mentioned methods, Joint Optimization \cite{tanaka2018joint} is the most relevant to our work. However, instead of learning the labels, we learn instance-dependent permutation layers. In addition, we are different in several other aspects: i) our method works well without any regularization terms and thus has fewer hyperparameters. ii) Unlike \cite{tanaka2018joint}, we are not limited to datasets with only few classes. iii) \cite{tanaka2018joint} is, in fact, a special case of our formulation.


\section{Methodology}
We develop a learning scheme that is robust to label noise. In our treatment, we consider c-class classification problems. Given a noisy training data    $D := \big\{(\mathbf{x}^i, \Tilde{y}^i)\big\}_{i=1}^{N}$ where $\mathbf{x}^i \in \mathbb{R}^m$ are the input features, $\Tilde{y}^i \in \{1, ..., c\}$ is the noisy label, and $N$ is the number of training samples, we aim to learn a classifier $f_\theta$ on $D$ while mitigating the impact of label noise.
\paragraph{Notation.} Consider a classifier $f_\theta:\mathbb{R}^m \rightarrow p(c)$, where $p(c)$ is a probability simplex over c classes, i.e.,  $p(c)= \{\mathbf{z} \in \mathbb{R}^c \, | \, \mathbf{z} \geq \mathbf{0}, \, \mathbf{1}^T\mathbf{z} = 1\}$. Let $P(i,j)$ be an elementary permutation matrix that permutes the $i^{th}$ and $j^{th}$ components of a vector; obtained by permuting the $i^{th}$ and $j^{th}$ rows of the \emph{identity} matrix. In addition, let $\Tilde{y}^k$ and $y^k$ be the noisy and the underlying clean label of the $k^{th}$ sample, respectively. We denote the \emph{softmax} function by $S$, i.e., $S(\mathbf{z})_i = e^{\mathbf{z}_i}/\sum_j e^{\mathbf{z}_j}$. Let $\boldsymbol{e}_i$ be a standard basis vector with 1 in the $i^{th}$ coordinate and $0$ elsewhere. Vectors are denoted as bold letters (e.g., $\mathbf{z}$). The maximum and minimum values in the vector $\mathbf{z}$ are represented as $\mathbf{z}_{max}$ and $\mathbf{z}_{min}$. We use \emph{superscript} to denote the sample index (e.g., $\mathbf{x}^k, \Tilde{y}^k \text{ and } \boldsymbol{\boldsymbol{\alpha}}^k$).

\subsection{Instance-Dependent Permutation Layer}
\label{permutation_matrix_section}
To combat the effect of label noise, we propose adding an instance-dependent permutation layer, sometimes referred to simply as the permutation layer. In this section, we introduce the notion of an instance-dependent permutation layer and demonstrate how it is constructed for a given training sample.

\vspace{8pt}
\begin{definition} \label{perm_mat_def}
    For a training sample $(\mathbf{x}^k, \Tilde{y}^k)$, we define the  instance-dependent permutation layer $\permmatrix{k}$ as
    $$\permmatrix{k} = \sum_{i=1}^{c} S(\boldsymbol{\alpha}^{k})_{i} \, P(\Tilde{y}^{k}, i)$$
\end{definition}

The permutation layer for the $k^{th}$ training sample $\permmatrix{k}$ is constructed as a convex combination of the permutation matrices in the set $\{ P(\Tilde{y}^{k},1), P(\Tilde{y}^{k},2), ...,  P(\Tilde{y}^{k},c)\}$. This set consists of all \emph{single-swap} permutation matrices involving the element $\Tilde{y}_{k}$. The convex combination is obtained using the learnable parameter $\boldsymbol{\alpha}^k$. Ultimately, we want to learn $\boldsymbol{\alpha}^k$ such that $\permmatrix{k}$ reflects the noise affecting the true label $y^k$, i.e., $\permmatrix{k} = P(\Tilde{y}^k, y^k)$. 

\begin{figure*}[t!]
\centering
\resizebox{17cm}{!}{
\includegraphics[width=1
\textwidth]{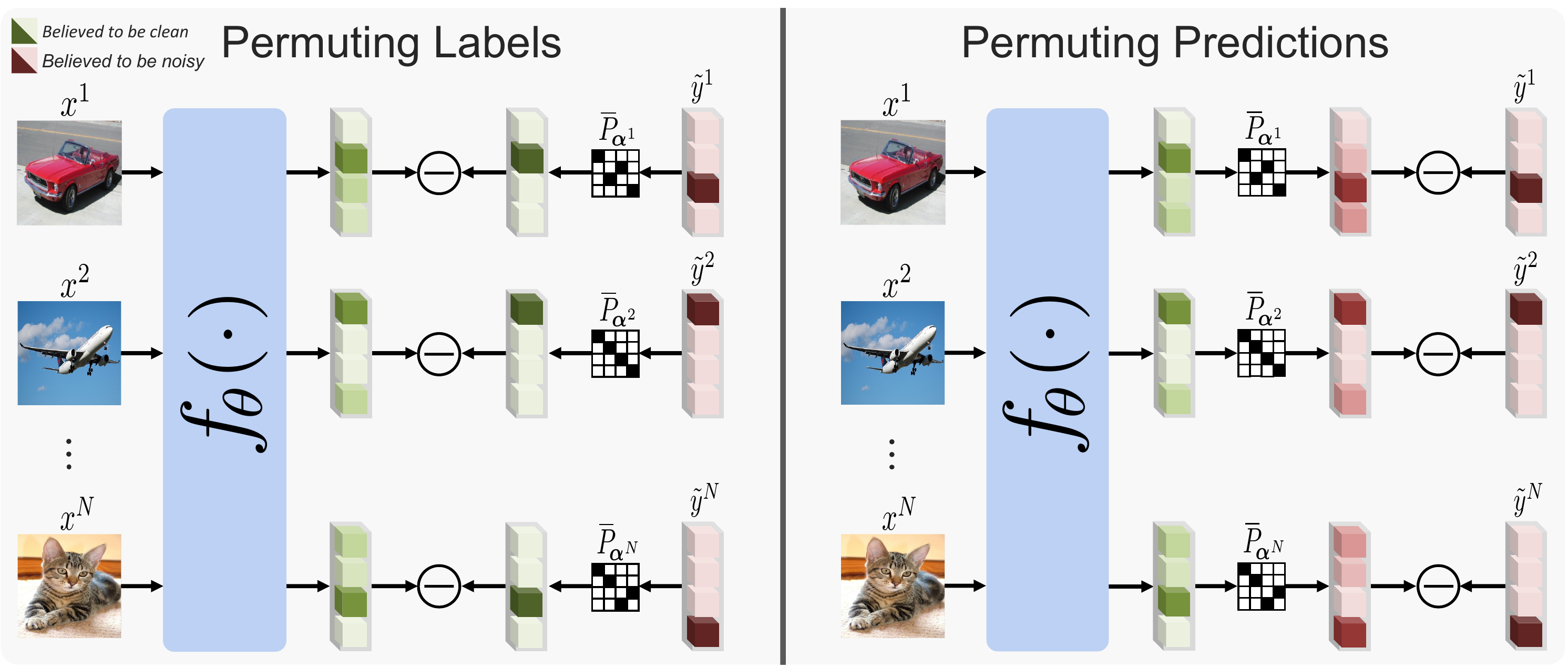}
}
\caption{The diagram illustrates the two variants of integrating the permutation layer. We assume that the noisy and underlying clean labels are: $\Tilde{y}^1=\boldsymbol{e}_3, y^1=\boldsymbol{e}_2, \Tilde{y}^2=y^2=\boldsymbol{e}_1, \Tilde{y}^N=\boldsymbol{e}_4, y^N=\boldsymbol{e}_3$, and that the instance-dependent permutation layers are learned successfully, i.e., $\permmatrix{k}=P(\Tilde{y}^k,y^k)$. When permuting labels, the loss is computed between the predictions and the corrected labels. However, when we instead permute the predictions, the loss is computed between the \emph{noisified} predictions $P(\Tilde{y}^k,y^k) \, f_\theta(x)$ and the noisy labels $\Tilde{y}^k$. Light and dark colors represent $0$ and $1$, respectively. }
\label{diagram}
\end{figure*}

\subsection{Permutation Layer Integration} \label{sec:Applying The Permutation Layer}
The proposed instance-dependent permutation layer alleviates the impact of label noise by correcting the loss. As demonstrated in Figure \ref{diagram}, we provide two approaches of incorporating the permutation layer into the learning procedure. Both approaches aim to capture the noise influencing $y^k$, i.e., $P(\Tilde{y}^k, y^k)$. The first approach applies the permutation layer to the noisy label $\boldsymbol{e}_{\Tilde{y}^k}$, as in equation \eqref{perm_on_label}. In this case, the permutation layer \emph{cleans} the noisy label, $P(\Tilde{y}^k, y^k) \, \boldsymbol{e}_{\Tilde{y}^k} = \boldsymbol{e}_{y^k}$. In the second approach, we instead apply the permutation layer to the model's prediction $f_\theta(\mathbf{x})$, as in equation \eqref{perm_on_pred}. Thus, the prediction $f_\theta(\mathbf{x}^k)$ is permuted to have a noise matching the noise on $y^k$, $P(\Tilde{y}^k, y^k)$. 
\begin{align}
    \pazocal{L}_1 = \frac{1}{N}\sum_{i=1}^{N}\ell(f_\theta(\mathbf{x}^i), \, \permmatrix{i} \, \boldsymbol{e}_{\Tilde{y}^i}) \label{perm_on_label} \\
    \pazocal{L}_2 = \frac{1}{N}\sum_{i=1}^{N}\ell(\permmatrix{i} f_\theta(\mathbf{x}^i), \, \boldsymbol{e}_{\Tilde{y}^i}) \label{perm_on_pred}
\end{align}

Note that these two approaches are analogues to the \emph{forward} and \emph{backward} corrections in~\cite{patrini2017making}. Looking closely at equation~\eqref{perm_on_label}, we can see that it is in fact equivalent to learning the label directly. In other words minimizing $\pazocal{L}_1$, with respect to $\theta$ and $\{\alpha^i\}_{i=1}^N$, is equivalent to minimizing $\sum_i \ell(f_\theta(\mathbf{x}^i), \, S(\boldsymbol{\alpha}^i))$. The approach of directly learning the label was first proposed in~\cite{tanaka2018joint}, and then adopted by \cite{yi2019probabilistic}. The following proposition demonstrates this equivalence by showing that applying a permutation to the label $\permmatrix{} \, \boldsymbol{e}_{\Tilde{y}^i}$ yields~$S(\boldsymbol{\alpha}^i)$.
\begin{prop} \label{equiv_of_permu_on_labels}
Given a training sample $(\mathbf{x}^k, \Tilde{y}^k)$, we have:
$$\permmatrix{k} \, \boldsymbol{e}_{\Tilde{y}^k} = S(\boldsymbol{\alpha}^k)$$
\end{prop}
\vspace{5pt}
\begin{proof}
$\permmatrix{k} \, \boldsymbol{e}_{\Tilde{y}^k} = \sum_{i=1}^{c} S(\boldsymbol{\alpha}^k)_{i} \, P(\Tilde{y}^k, i) \, \boldsymbol{e}_{\Tilde{y}^k}$
\\
\begin{center}
$ \;\;\,\, = \sum_{i=1}^{c} S(\boldsymbol{\alpha}^k)_{i} \, \boldsymbol{e}_{i} \, = \, S(\boldsymbol{\alpha}^k)$
\end{center}\end{proof}
\vspace{2pt}

In the next section, we further analyze these two approaches, and analytically show that they possess different theoretical properties. 

\subsection{Permutation Layer Properties: Theoretical Analysis}
Although the two approaches of employing the permutation layer seem equivalent, they are in fact substantially different. In this section, we compare these approaches and analytically show that applying the permutation to the prediction has two key advantages: ($i$) The optimization does not get stuck in a local optimum after solving for the permutation layer parameters $\boldsymbol{\alpha}$. ($ii$) The norm of the gradient $||\nabla_{\!\boldsymbol{\alpha}} \pazocal{L}_2||$ is affected by the model's confidence, i.e., $\pazocal{L}_2$ is ``confidence-aware". 

\paragraph{Getting stuck in a local minimum.} 
As noted by \cite{tanaka2018joint}, alternating minimization of the KL-divergence loss with respect to the labels and the model parameters suffers from premature convergence to local optima. This stems from the fact that the minimum of the loss (i.e., $\pazocal{L} = 0$) is attained immediately after solving with respect to the labels.
In Proposition \ref{perm_on_label_gets_stuck}, we build on the observation in \cite{tanaka2018joint} and show that for \emph{any} loss function satisfying assumption \ref{loss_assumption}, the minimization of $\pazocal{L}_1$ gets stuck in a local minimum, i.e., the minimum is attained when solving for $\boldsymbol{\alpha}$. On the other hand, Proposition \ref{perm_on_pred_does_not_gets_stuck} shows that minimizing $\pazocal{L}_2$ does not stop prematurely, i.e., the minimum is \emph{not} attained when solving for~$\boldsymbol{\alpha}$.

\begin{assumption}\label{loss_assumption}
    Let $\ell(\mathbf{p},\mathbf{q}): p(c) \times p(c) \rightarrow \mathbb{R}^+$ , where $\, \mathbb{R}^+$ is the set $ \{ x \in \mathbb{R} | x \geq 0 \}$. We assume that $\ell(\mathbf{p},\mathbf{q})=0 \,$  if and only if  $ \, \mathbf{p} = \mathbf{q}$ .
\end{assumption}

\paragraph{Implications of Assumption \ref{loss_assumption}.} This assumption includes most well known loss functions such as p-norm loss functions and KL-divergence. However, the cross-entropy loss generally does not satisfy this assumption (unless $\max_{i} \mathbf{q}_i = 1$, which is the case in Proposition \ref{perm_on_pred_does_not_gets_stuck}).
\vspace{3pt}

\begin{prop}
\label{perm_on_label_gets_stuck}
Let $\ell$ be a loss function satisfying assumption \ref{loss_assumption}. Then, we have
\begin{align*}
    \ell(f_{\theta}(\mathbf{x}), \xoverline{P}_{\!\!\boldsymbol{\alpha}^{\ast}(\theta)} \, \boldsymbol{e}_{\Tilde{y}}) = 0
\end{align*}
where $\boldsymbol{\alpha}^{\ast}(\theta) = \underset{\boldsymbol{\alpha}}{\arg\min} \, \ell( f_{\theta}(\mathbf{x}), \, \permmatrix{} \, \boldsymbol{e}_{\Tilde{y}})$ 
\end{prop}
\begin{proof}
    From proposition \ref{equiv_of_permu_on_labels}, we have $\permmatrix{} \, \boldsymbol{e}_{\Tilde{y}} = S(\boldsymbol{\alpha})$. Then, 
    \begin{align*}
        &\boldsymbol{\alpha}^{\ast}(\theta) = \underset{\boldsymbol{\alpha}}{\arg\min} \, \ell( f_{\theta}(\mathbf{x}), \, S(\boldsymbol{\alpha})) \\
        \Rightarrow \; &S\big(\boldsymbol{\alpha}^{\ast}(\theta)\big) = f_{\theta}(\mathbf{x})
    \end{align*}
    Finally, by assumption \ref{loss_assumption} we have $$\ell(f_{\theta}(\mathbf{x}), \xoverline{P}_{\!\!\boldsymbol{\alpha}^{\ast}(\theta)} \, \boldsymbol{e}_{\Tilde{y}}) = \ell(f_{\theta}(\mathbf{x}), f_{\theta}(\mathbf{x})) = 0 $$
\end{proof}

\begin{prop} \label{perm_on_pred_does_not_gets_stuck}
Let $\ell$ be a loss function satisfying assumption \ref{loss_assumption}, and $\theta_i$ be some choice for the model's parameter such that $\underset{j}{max} \, [f_{\theta_i}(\mathbf{x})]_j < 1$. Then, we have
\begin{align*}
    \ell(\xoverline{P}_{\!\!\boldsymbol{\alpha}^{\ast}(\theta_i)} \, f_{\theta_i}(\mathbf{x}), \, \boldsymbol{e}_{\Tilde{y}}) > 0
\end{align*}
where $\boldsymbol{\alpha}^{\ast}(\theta_i) = \underset{\boldsymbol{\alpha}}{\arg\min} \, \ell( \permmatrix{} \, f_{\theta_i}(\mathbf{x}), \, \boldsymbol{e}_{\Tilde{y}})$ 
\end{prop}
\begin{proof}
We start by showing that the two arguments of $\ell$ are not equal. It is enough to show that $\big[\xoverline{P}_{\!\!\boldsymbol{\alpha}^{\ast}(\theta_i)} \, f_{\theta_i}(\mathbf{x})\big]_{\Tilde{y}} < [\boldsymbol{e}_{\Tilde{y}}]_{\Tilde{y}} = 1$.
\begin{align*}
    &\big[\xoverline{P}_{\!\!\boldsymbol{\alpha}^{\ast}(\theta_i)} \, f_{\theta_i}(\mathbf{x})\big]_{\Tilde{y}} = \boldsymbol{e}_{\Tilde{y}}^T \, \xoverline{P}_{\!\!\boldsymbol{\alpha}^{\ast}(\theta_i)} \, f_{\theta_i}(\mathbf{x}) \\ &= \boldsymbol{e}_{\Tilde{y}}^T \, \sum_{j=1}^{C} S\big(\boldsymbol{\alpha}^\ast(\theta_i)\big)_{j} \, P(\Tilde{y}, j) \, f_{\theta_i}(\mathbf{x}) \\
    &= \sum_{j=1}^{C} S\big(\boldsymbol{\alpha}^\ast(\theta_i)\big)_{j} \, \boldsymbol{e}_j^T \, f_{\theta_i}(\mathbf{x}) \\ &= S\big(\boldsymbol{\alpha}^\ast(\theta_i)\big)^T \, f_{\theta_i}(\mathbf{x}) \leq \underset{j}{\max} \, [f_{\theta_i}(\mathbf{x})]_j < 1.
\end{align*}
Therefore, $$\xoverline{P}_{\!\!\boldsymbol{\alpha}^{\ast}(\theta_i)} \, f_{\theta_i}(\mathbf{x}) \neq \boldsymbol{e}_{\Tilde{y}}$$
Finally, by assumption \ref{loss_assumption}, $ \, \ell(\xoverline{P}_{\!\!\boldsymbol{\alpha}^{\ast}(\theta_i)} \, f_{\theta_i}(\mathbf{x}), \, \boldsymbol{e}_{\Tilde{y}}) > 0$
\end{proof}

\vspace{10pt}

\noindent Note that for classifiers with a softmax layer, i.e., $f_{\theta_i}(\mathbf{x}) = S(h_{\theta_i}(\mathbf{x}))$, $\,\underset{j}{\max} \, [f_{\theta_i}(\mathbf{x})]_j < 1\,$ for all $\theta_i$.

\begin{figure*}[t!]
\centering
\includegraphics[width=1
\textwidth]{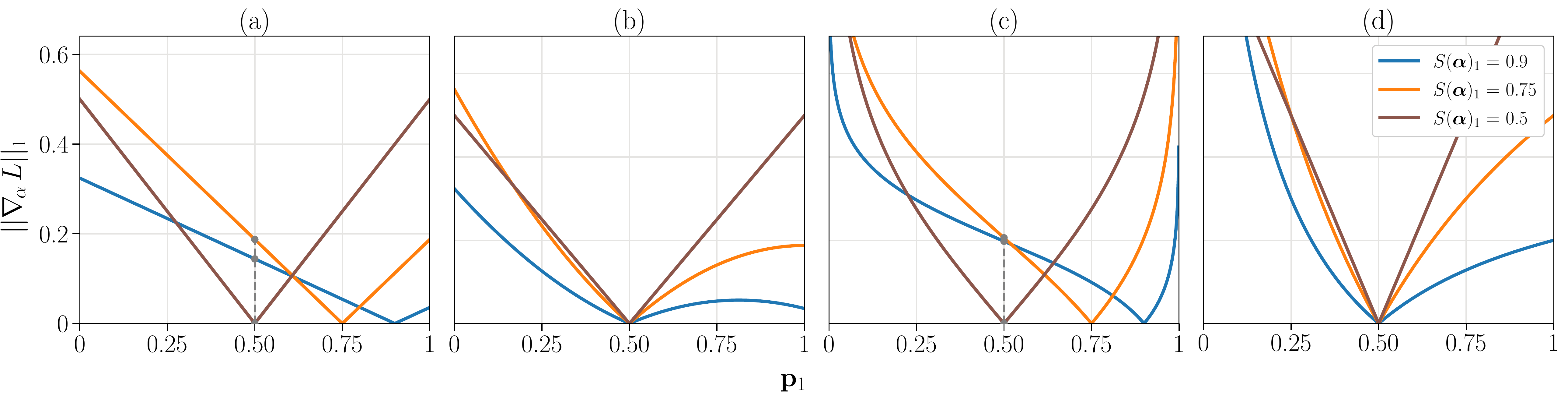} 
\caption{The relationship between $|| \nabla_{\!\boldsymbol{\alpha}} L ||_1$ and the confidence of the first class $\mathbf{p}_1$ over different choices of $\mathbf{\boldsymbol{\alpha}}$ and $\Tilde{y}=1$.
Figures (a,b) use the $L_2$ norm whereas (c,d) use the KL-divergence. In figures (a,c) the permutation layer is applied to the label, i.e., $L=\ell(\mathbf{p}, \permmatrix{}\boldsymbol{e}_{\Tilde{y}})$. In figures (b,d) the permutation is applied to the prediction, i.e., $L=\ell(\permmatrix{}\mathbf{p}, \boldsymbol{e}_{\Tilde{y}})$. This analysis shows that \emph{only} when the permutation layer is applied to the prediction, $|| \nabla_{\!\boldsymbol{\alpha}} L ||_1$ decreases as the confidence decrease and is zero when $\mathbf{p}=[0.5, 0.5]$. However, for the case of applying the permutation to the label, the gradient is zero only when $\mathbf{p} = S(\boldsymbol{\alpha})$.}
\label{Confidence_aware_fig}
\end{figure*}

\paragraph{Confidence-aware.} Another advantage of applying the permutation to the prediction is that the gradient norm $||\nabla_{\!\boldsymbol{\alpha}} \pazocal{L}_2||$ is proportional to the confidence of the classifier's prediction, as in Proposition \ref{Confidence_aware_prop}. With this approach, unconfident predictions cause minimal updates on $\boldsymbol{\alpha}$, while confident ones prompt more notable updates. We measure the confidence of a prediction $f_\theta(\mathbf{x})$ using $f_\theta(\mathbf{x})_{max}-f_\theta(\mathbf{x})_{min}$. We believe this behavior to be generally desirable, especially for classifiers that provide \emph{well-calibrated} predictions during training \cite{thulasidasan2019mixup}. In well-calibrated classifiers, the predicted scores are indicative of the actual likelihood of correctness \cite{guo2017calibration}.

\begin{assumption}
\label{loss_assumption_2}
    Let $\ell(\mathbf{p},\mathbf{q}): p(c) \times p(c) \rightarrow \mathbb{R}^+$ , where $\, \mathbb{R}^+$ is the set $ \{ x \in \mathbb{R} | x \geq 0 \}$, be a loss function satisfying $||\nabla_{\!\mathbf{p}} \, \ell(\mathbf{p},\mathbf{q})||_1 \leq M$ when $\, \mathbf{p}_{max} - \mathbf{p}_{min} <~1/c$ .
    
    \vspace{5pt}
    \noindent Where $M \in \mathbb{R}^+$, $\mathbf{p}$ and $\mathbf{q}$ represent the prediction and the target, respectively.
\end{assumption}

\paragraph{Remark on Assumption \ref{loss_assumption_2}} This assumption is clearly satisfied for $p$-\emph{norm} loss functions with $p\geq1$. It is also satisfied for the cross-entropy loss and KL-divergence; the proof is provided in the \emph{supplementary material}.

\vspace{5pt}

\begin{prop}
\label{Confidence_aware_prop}
Let $\ell$ be a loss function satisfying assumption \ref{loss_assumption_2} and the loss for a given training sample $(\mathbf{x}^k, \Tilde{y}^k)$ be $L=\ell(\permmatrix{k} f_\theta(\mathbf{x}^k), \, \boldsymbol{e}_{\Tilde{y}^k})$. Then, $\nabla_{\!\boldsymbol{\alpha}^k} L$ satisfies:
\begin{align*}
    || \nabla_{\!\boldsymbol{\alpha}^k} L ||_1\leq \frac{c \, M}{4} \big[f_\theta(\mathbf{x}^k)_{max}-f_\theta(\mathbf{x}^k)_{min}\big] \\ \text{if} \quad \big[f_\theta(\mathbf{x})_{max}-f_\theta(\mathbf{x})_{min}\big] < 1/c \,.
\end{align*}
where $c$ is the number of classes.
\end{prop}
\begin{proof}
We start by taking the partial derivative of $L$ with respect to $\boldsymbol{\alpha}^k_j$, where $1\leq j \leq c$. To simplify notation, let $\Tilde{f} = \permmatrix{k} f_\theta(\mathbf{x}^k)$
\begin{align*}
    \Big|\frac{\partial L}{\partial \boldsymbol{\alpha}^k_j}\Big| &= \Big|\nabla_{\!\Tilde{f}}\, L \boldsymbol{\cdot} \frac{\partial \Tilde{f}}{\partial \boldsymbol{\alpha}^k_j}\Big| \\ & \leq \Big|\nabla_{\!\Tilde{f}} \, L\Big| \boldsymbol{\cdot} \Big|\frac{\partial \Tilde{f}}{\partial \boldsymbol{\alpha}^k_j}\Big| \tag{3.1}\label{eq:3.1}
\end{align*}
Now, we obtain an upper bound for $\Big|\frac{\partial \Tilde{f}}{\partial \boldsymbol{\alpha}^k_j}\Big|$
\begin{align*}
    \Big|\frac{\partial \Tilde{f}}{\partial \boldsymbol{\alpha}^k_j}\Big| &= \Big|\sum_{i=1}^{c} \frac{\partial S(\boldsymbol{\alpha}^k)_{i}}{\partial \boldsymbol{\alpha}^k_j} P(\Tilde{y}^k,i) \; f_\theta(\mathbf{x}^k)\Big| \\
    &= \Big|\sum_{\substack{i=1 \\ i \neq j}}^{c} S(\boldsymbol{\alpha}^k)_j S(\boldsymbol{\alpha}^k)_i [P(\Tilde{y}^k,j) f_\theta(\mathbf{x}^k) \\[-13pt]
    & \qquad \qquad \quad -  P(\Tilde{y}^k,i) f_\theta(\mathbf{x}^k)]\Big| 
\end{align*}
Using the triangle inequality and the fact that $\big|P(\Tilde{y}^k,j) f_\theta(\mathbf{x}^k) -  P(\Tilde{y}^k,i) f_\theta(\mathbf{x}^k)\big| \leq \big[f_\theta(\mathbf{x}^k)_{max}-f_\theta(\mathbf{x}^k)_{min}\big] \, \mathbf{1}\,$, we have
\begin{align*}
    \Big|\frac{\partial \Tilde{f}}{\partial \boldsymbol{\alpha}^k_j}\Big| & \leq \sum_{\substack{i=1 \\ i \neq j}}^{c} S(\boldsymbol{\alpha}^k)_j S(\boldsymbol{\alpha}^k)_i \big[f_\theta(\mathbf{x}^k)_{max}-f_\theta(\mathbf{x}^k)_{min}\big] \mathbf{1} \\
    &= S(\boldsymbol{\alpha}^k)_j (1-S(\boldsymbol{\alpha}^k)_j) \big[f_\theta(\mathbf{x}^k)_{max}-f_\theta(\mathbf{x}^k)_{min}\big] \mathbf{1} \\
    & \leq \frac{1}{4} \big[f_\theta(\mathbf{x}^k)_{max}-f_\theta(\mathbf{x}^k)_{min}\big] \mathbf{1} \tag{3.2}\label{eq:3.2}
\end{align*}

\noindent Combining $\eqref{eq:3.1}$ and $\eqref{eq:3.2}$, the partial derivative of $L$ is
\begin{align*}
    \Big|\frac{\partial L}{\partial \boldsymbol{\alpha}^k_j}\Big| & \leq \frac{1}{4} \big[f_\theta(\mathbf{x}^k)_{max}-f_\theta(\mathbf{x}^k)_{min}\big] \, \big|\big| \nabla_{\!\Tilde{f}} \, L \big|\big|_1
\end{align*}

\noindent By Assumption \ref{loss_assumption_2}, we conclude that
\begin{align*}
    || \nabla_{\!\boldsymbol{\alpha}^k} L ||_1 &\leq \frac{c}{4} \big[f_\theta(\mathbf{x}^k)_{max}-f_\theta(\mathbf{x}^k)_{min}\big] \, \big|\big| \nabla_{\!\Tilde{f}} \, L \big|\big|_1 \\
    &\leq \frac{c \, M}{4} \big[f_\theta(\mathbf{x}^k)_{max}-f_\theta(\mathbf{x}^k)_{min}\big]
\end{align*} 
\end{proof}

The result in proposition \ref{Confidence_aware_prop} can be motivated intuitively by the following example.

\paragraph{Example:}
Consider a training sample $(\mathbf{x}^k, \Tilde{y}^k)$ where the classifier's prediction has equal probabilities $f(\mathbf{x}^k)= ~ \frac{1}{c} ~ \mathbf{1}$, or equivalently $f(\mathbf{x}^k)_{max}-f(\mathbf{x}^k)_{min} = 0$. Then, it is easy to see that $\ell(\permmatrix{1}f(\mathbf{x}^k),\, \Tilde{y}^k) = \ell(\permmatrix{2}f(\mathbf{x}^k),\, \Tilde{y}^k)$ for all $\boldsymbol{\alpha}^1, \boldsymbol{\alpha}^2 \in \mathbb{R}^c$. This is obvious since applying the permutation layer on $\sfrac{1}{c} ~ \mathbf{1}$ does not change its value, e.g.,
$\permmatrix{1}\, \frac{1}{c} \, \mathbf{1} = \sum_{i=1}^{c} S(\boldsymbol{\alpha}^{1})_{i} \, P(\Tilde{y}^{k}, i) \, \frac{1}{c} \, \mathbf{1} = \sum_{i=1}^{c} S(\boldsymbol{\alpha}^{1})_{i} \, \frac{1}{c} \, \mathbf{1} =\frac{1}{c} \, \mathbf{1}$
Finally, we can conclude that $\nabla_{\!\boldsymbol{\alpha}} \, \ell(\permmatrix{} f(\mathbf{x}^k), \Tilde{y}^k) = \mathbf{0}$.

\vspace{3pt}

Figure \ref{Confidence_aware_fig} illustrates the confidence-aware property of $\pazocal{L}_2$ by plotting the relationship between $|| \nabla_{\!\boldsymbol{\alpha}} L ||_1$ and the confidence of the first class $\mathbf{p}_1$, in a binary classification problem. This also shows that $\pazocal{L}_1$ lacks this property.

\subsection{Permutation Layer Learning}
We simultaneously learn the model's parameters $\theta$ and the parameters of the permutation layer $\{ \boldsymbol{\alpha}^i \}_{i=1}^{N}$ using gradient descent.
In fact, the optimization problem has a closed-form solution with respect to  $\{ \boldsymbol{\alpha}^i \}_{i=1}^{N}$. However, similar to \cite{yi2019probabilistic}, we use gradient descent to allow for gradual changes of the parameters and to take advantage of the confidence-aware property described in proposition \ref{Confidence_aware_prop}. The permutation layer parameters, $\{ \boldsymbol{\alpha}^i \}_{i=1}^{N}$, are updated as follows:
\[
\boldsymbol{\alpha}^i[t+1] \gets \boldsymbol{\alpha}^i[t] - \myPermLR{} \nabla_{\!\boldsymbol{\alpha}^i} \pazocal{L}
\]
for $1 \leq i \leq N$, where $\myPermLR{}$ is the learning rate of the permutation layer.
The initialization of the permutation layer weights is determined by the hyperparameter $\alphaInit{}$. 
Precisely, $\alphaInit{}$ affects the initialization as follows:  
\[
S({\boldsymbol{\alpha}}^i)_j = 
\left\{
    \begin{array}{ll}
        \alphaInit{}, & \text{if } j=\Tilde y^i \\[1em]
        \frac{1-\alphaInit{}}{c-1} , & \text{if } j\neq \Tilde y^i
    \end{array}
\right.,
\]
where $\frac{1}{c}<\alphaInit{}<1$. We avoid setting $\alphaInit{}$ very close to 1 to  allow for a more desirable gradient flow, i.e., avoiding vanishing gradients. The effects of the choices of  $\myPermLR{}$ and $\alphaInit{}$ are analyzed in section \ref{sections:hanalysis}. Note that the permutation layer is discarded at inference time.

In the previous sections, we have established the similarity between permuting the labels and Joint Optimization \cite{tanaka2018joint}, and that permuting the predictions has more appealing theoretical properties. Therefore, in the next section, we empirically analyze PermLL when the permutation is applied to the prediction.

\begin{table*}[h!]
\centering
\resizebox{0.80900\textwidth}{!}{%
\begin{tabular}{clcccccc}
\hline
\multirow{2}{*}{\textbf{Dataset}} & \multicolumn{1}{c}{\multirow{2}{*}{\textbf{Method}}} & \multicolumn{4}{c}{\textit{Symmetric noise}} & \multicolumn{2}{c}{\textit{Aysmmetric noise}} \\
                                   &                 & 20\%         & 40\%         & 60\%         & \multicolumn{1}{c|}{80\%}         & 20\%         & 40\%         \\ \cline{1-8} 
\multirow{7}{*}{\textbf{CIFAR10}}  & CE              & 86.98 ± 0.12 & 81.88 ± 0.29 & 74.14 ± 0.56 & \multicolumn{1}{c|}{53.82 ± 1.04} & 88.59 ± 0.34 & 80.11 ± 1.44 \\
                                   & Bootstrap \cite{reed2014training}      & 86.23 ± 0.23 & 82.23 ± 0.37 & 75.12 ± 0.56 & \multicolumn{1}{c|}{54.12 ± 1.32} & 88.26 ± 0.24 & 81.21 ± 1.47 \\
                                   & Forward \cite{patrini2017making}        & 87.99 ± 0.36 & 83.25 ± 0.38 & 74.96 ± 0.65 & \multicolumn{1}{c|}{54.64 ± 0.44} & 89.09 ± 0.47 & 83.55 ± 0.58 \\
                                   & GCE    \cite{zhang2018generalized}         & 89.83 ± 0.20 & 87.13 ± 0.22 & 82.54 ± 0.23 & \multicolumn{1}{c|}{64.07 ± 1.38} & 89.33 ± 0.17 & 76.74 ± 0.61 \\
                                   & Joint \cite{tanaka2018joint}           & 92.25 & 90.79 & 86.87 & \multicolumn{1}{c|}{69.16} & - & - \\
                                   & SL \cite{wang2019symmetric}             & 89.83 ± 0.32 & 87.13 ± 0.26 & 82.81 ± 0.61 & \multicolumn{1}{c|}{68.12 ± 0.81} & 90.44 ± 0.27 & 82.51 ± 0.45 \\
                                   & ELR \cite{liu2020early}          & 91.16 ± 0.08 & 89.15 ± 0.17 & 86.12 ± 0.49 & \multicolumn{1}{c|}{73.86 ± 0.61} & \textbf{93.52 ± 0.23} & 90.12 ± 0.47 \\
                                   & SELC \cite{SELC}           & 93.09 ± 0.02 & 91.18 ± 0.06 & 87.25 ± 0.09 & \multicolumn{1}{c|}{74.13 ± 0.14} & - & \textbf{91.05 ± 0.11} \\
                                   & \textbf{PermLL} & \textbf{93.17 ± 0.11} & \textbf{91.30 ± 0.23} & \textbf{87.94  ± 0.21} & \multicolumn{1}{c|}{\textbf{75.83  ± 0.56}}        & 92.16 ± 0.15 & 86.09 ± 0.23         \\ \hline
\multirow{7}{*}{\textbf{CIFAR100}} & CE              & 58.72 ± 0.26 & 48.20 ± 0.65 & 37.41 ± 0.94 & \multicolumn{1}{c|}{18.10 ± 0.8}  & 59.20 ± 0.18 & 42.74 ± 0.61 \\
                                   & Bootstrap \cite{reed2014training}      & 58.27 ± 0.21 & 47.66 ± 0.55 & 34.68 ± 1.1  & \multicolumn{1}{c|}{21.64 ± 0.97} & 62.14 ± 0.32 & 45.12 ± 0.57 \\
                                   & Forward \cite{patrini2017making}     & 39.19 ± 2.61 & 31.05 ± 1.44 & 19.12 ± 1.95 & \multicolumn{1}{c|}{8.99 ± 0.58}  & 42.46 ± 2.16 & 34.44 ± 1.93 \\
                                   & GCE \cite{zhang2018generalized}         & 66.81 ± 0.42 & 61.77 ± 0.24 & 53.16 ± 0.78 & \multicolumn{1}{c|}{29.16 ± 0.74} & 66.59 ± 0.22 & 47.22 ± 1.15 \\
                                   & Joint \cite{tanaka2018joint}          & 58.15 & 54.81 & 47.94 & \multicolumn{1}{c|}{17.18} & - & -\\
                                   & Pencil \cite{yi2019probabilistic}         & 73.86 ± 0.34 & 69.12 ± 0.62 & 57.70 ± 3.86 & \multicolumn{1}{c|}{fail} & - & 63.61 ± 0.23\\
                                   & SL \cite{wang2019symmetric}             & 70.38 ± 0.13 & 62.27 ± 0.22 & 54.82 ± 0.57 & \multicolumn{1}{c|}{25.91 ± 0.44} & 72.56 ± 0.22 & 69.32 ± 0.87 \\
                                   & ELR    \cite{liu2020early}         & 74.21 ± 0.22 & 68.28 ± 0.31 & 59.28 ± 0.67 & \multicolumn{1}{c|}{29.78 ± 0.56} & \textbf{74.03 ± 0.31} & \textbf{73.26 ± 0.64} \\
                                   & SELC  \cite{SELC}          & 73.63 ± 0.07 & 68.46 ± 0.10 & 59.41 ± 0.06 & \multicolumn{1}{c|}{32.63 ± 0.06} & - &  70.82 ± 0.09 \\
                                   & \textbf{PermLL} & \textbf{74.35  ± 0.34} & \textbf{71.37  ± 0.36} & \textbf{65.58 ± 0.37} & \multicolumn{1}{c|}{\textbf{45.81 ± 0.49}} & 72.66 ± 0.13 & 51.60 ± 0.11 \\
    \hline
\end{tabular}%

        }

\caption{
Results on CIFAR10 and CIFAR100 on symmetric and asymmetric label noise with different noise levels. We report the mean and standard deviation test accuracy of our method over three random initializations. All methods use Resnet-34. Results for SL, GCE, Forward, Bootstrap are taken from \cite{liu2020early}, results from Joint are taken from \cite{SELC}, and other results are taken from their respective papers. The best results are highlighted in bold. 
}
\label{Table:Exp1}
\end{table*}

\begin{table}[]
\resizebox{\columnwidth}{!}{%
\begin{tabular}{@{}lclc@{}}
\toprule
\multicolumn{1}{c}{\textbf{Method}} & \textbf{Result} & \multicolumn{1}{c}{\textbf{Method}} & \textbf{Result} \\ \midrule
CE      & 69.10 & ELR+* \cite{liu2020early}          & 74.81 \\
Forward \cite{reed2014training} & 69.84 & DivideMix* \cite{li2020dividemix}  & 74.76 \\
Pencil  \cite{yi2019probabilistic}& 73.49 & UniCon* \cite{karim2022unicon} & 74.98 \\
JNPL  \cite{kim2021joint}& 74.15 & MLNT  \cite{li2019learning} & 73.47 \\
Joint \cite{tanaka2018joint}& 72.16 & \textbf{PermLL} & \textbf{74.99} \\ \bottomrule
\end{tabular}%
}
\caption{Test accuracy results on Clothing-1M compared to state-of-the-art methods. PermLL outperforms all other methods despite it's simplicity. Methods denoted with * use heavy augmentations (e.g mixup), semi-supervised training and multiple networks to achieve their results}
\label{Table:Clothing}
\end{table}

\section{Experiments}
\subsection{Datasets}
To evaluate the effectiveness of PermLL, we use three standard image classification benchmarks: CIFAR-10, CIFAR-100 \cite{CIFAR} and Clothing-1M \cite {CLOTH}. 
Since the CIFAR10/100 datasets are known to be clean, we simulate a noisy learning setting by injecting synthetic noise into the training labels. Two types of synthetic label noise are used: symmetric and asymmetric noise. For the symmetric noise, we take a fraction of the labels and flip them to one of the c classes with a uniform random probability. Asymmetric noise attempts to mimic the mistakes that naturally occur by annotators in the real world. We follow \cite{Patrini_2017_CVPR} for injecting the asymmetric noise in CIFAR-10: truck$\,\to\,$autombile, deer$\,\to\,$horse, bird$\,\to\,$airplane, cat $\leftrightarrow$ dog. In the case of asymmetric noise for CIFAR-100, we flip labels of subclasses within randomly selected superclasses cyclically with a total of 20 superclasses. Following \cite{liu2020early, tanaka2018joint}, we hold out 10\% of the training data as a validation set.

Clothing1M, on the other hand, is a large dataset of images from online clothing shops containing realistic noise. The noise level is estimated to be around 38.5\% \cite{CLOTH}. This dataset contains four sets: a noisy training set (1M samples), a clean training set (50K samples), a clean validation set (14K samples), and a clean testing set (10K samples). In our experiments, we completely exclude the 50k samples in the clean training set as PermLL does not assume the availability of a clean training set.  

\vspace{5pt}

\begin{figure*}
    \centering
    \includegraphics[width=1
    \textwidth]{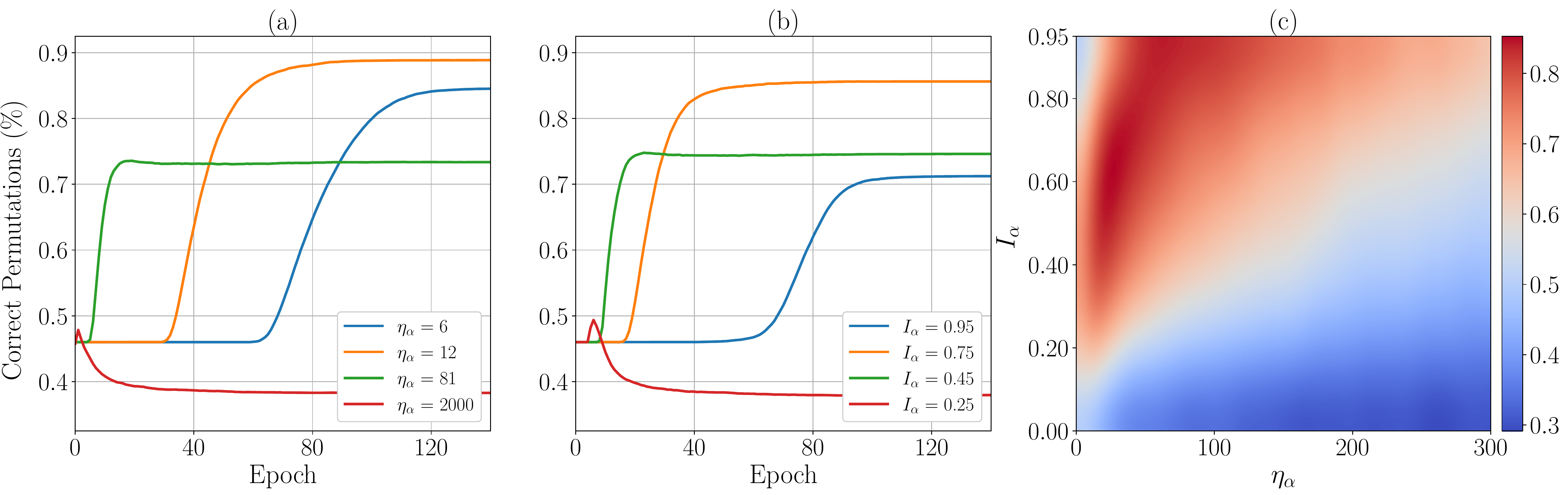}
    \caption{
    The effect of hyperparameter choices on the percentage of correctly learned permutation layers on CIFAR-10 with 60\% symmetric noise.
    (a) Shows the effect of changing the permutation learning rate $\myPermLR{}$ on the percentage of correct permutations when $\alphaInit{}$ = 0.75.
    (b) Shows the effect of changing $\alphaInit{}$ when $\myPermLR{}$ is set to 24. 
    (c) Shows a smoothed heatmap of the percentage of correctly learned permutations at the end of training as a function of $\myPermLR{}$ and $\alphaInit{}$. 
    }
    \label{fig:lr_alpha_on_corrected}
\end{figure*}

\subsection{Implementation}
\label{section:Implementation}

For CIFAR-10 we use a ResNet34 \cite{he2016deep} and an SGD optimizer with learning rate of 0.02, momentum of 0.9, a weight decay of 0.0005 and a batch size 128. The model is trained for 120 epochs and the learning rate is decayed by a factor of 10 at epochs 80 and 100.
Additionally, we apply standard augmentation and data preprocessing: padding followed by random cropping, random horizontal flip, and data normalization, similar to \cite{liu2020early, tanaka2018joint, yi2019probabilistic}. The same setup is used for CIFAR-100 except that the weight decay used is 0.001, and the learning rate decays once by a factor of 10 at epoch 100.
For all noise settings in CIFAR10 we set $\alphaInit{} = 0.35$, and $\myPermLR{} = 1.5$. For CIFAR100 we use $\alphaInit{} = 0.225$, with $\myPermLR{} = 3$ for symmetric noise and $\myPermLR{} = 6$ for asymmetric noise.

For Clothing1M we use a pretrained Resnet50 from TorchVision \cite{NEURIPS2019_9015} and an SGD optimizer with learning rate of 0.001, momentum of 0.9, a weight decay of 0.001, and a batch size of 64. The model is trained for 15 epochs and the learning rate is decayed by a factor of 10 at epochs 5, 10 and 13, with $\alphaInit{} = 0.275$ and $\myPermLR{} = 80$. 
Standard augmentation and data preprocessing are applied; the images are normalized, resized to 256x256, randomly cropped to 224x224, and randomly horizontally flipped.
Similar to previous methods \cite{liu2020early, li2020dividemix}, during training, we sample balanced batches (based on the noisy labels) as Clothing1M contains significant class imbalance. All experiments were implemented using PyTorch \cite{NEURIPS2019_9015} and were run on NVIDIA V100 and NVIDIA RTX A6000 GPUs.




\subsection{Results}
Table \ref{Table:Exp1} shows the results of PermLL on CIFAR10/100 with symmetric and asymmetric synthetic noise. We compare our results with state-of-the-art methods that use a similar setup (e.g., same architecture and augmentation). To ensure fairness, we use the same architecture, batch size, and augmentation as other baselines when possible (as described in \ref{section:Implementation}). PermLL consistently outperforms all other methods in symmetric noise. This is especially clear in the high noise CIFAR100 setting, in which PermLL achieves an improvement of $\sim$6\% and $\sim$12\% in terms of test accuracy for 60\% and 80\% symmetric noise, respectively, while using the same hyperparameters for all noise percentages. In the case of simulated asymmetric noise, our method attains comparable performance at 20\% noise. In the 40\% simulated asymmetric noise, our method achieves poor performance compared to the baselines. 
However, we observe that PermLL does not face the same issue when faced with real asymmetric instance-dependent noise. 

To validate the efficacy of PermLL in real noise environments, we use the Clothing-1M dataset as a benchmark. We report the test accuracy in Table \ref{Table:Clothing} of PermLL and compare it with state-of-the-art methods. We achieve high performance and outperform all baselines, slightly outperforming UniCon. It is worth noting that the comparison might not be entirely fair as methods denoted with (*), unlike PermLL, use semi-supervised training, heavy augmentation, and multiple networks to achieve their results. Nevertheless, PermLL remarkably achieves very competitive results despite its relative simplicity, demonstrating its practicality and strength.


\subsection{Hyperparameter Analysis}\label{sections:hanalysis}
In this section, we analyze the effect of PermLL hyperparameters on learning correct permutation layers, similar to \cite{tanaka2018joint, yi2019probabilistic}.
Recall that, we ultimately want to learn $\boldsymbol{\alpha}^i$ such that $\permmatrix{i} = P(\Tilde{y}^i, y^i)$ for all $i$, where $\Tilde{y}^i$ and $y^i$ are the clean and noisy labels, respectively.
Therefore, we consider a permutation layer to be correct if $\arg\max_j \boldsymbol{\alpha}^i_j = y^i$ for all $i$.

Figure \ref{fig:lr_alpha_on_corrected} (a) shows the impact of changing the permutation layer's learning rate $\myPermLR{}$ on the percentage of correct permutations throughout training.
A small $\myPermLR{}$ delays or halts learning the permutation layers, causing the model to memorize the noisy labels.
In contrast, a large $\myPermLR{}$ results in early and erroneous permutations. 
Furthermore, an extremely large $\myPermLR{}$ reduces the percentage of correct permutations below its initial value, as early wrong predictions cause harmful large updates on the permutation layer parameters. Similarly in Figure \ref{fig:lr_alpha_on_corrected} (b), we study the effect of changing $\alphaInit{}$. The effect appears to be inversely similar to the effect of changing $\myPermLR{}$.
Analogous to the case of a high $\myPermLR{}$, extremely small $\alphaInit{}$ significantly reduces the permutation accuracy. Figure \ref{fig:lr_alpha_on_corrected} (c) illustrates the relationship between $\alphaInit{}$ and $\myPermLR{}$, showing their influence on the percentage of correct permutations over a larger scale. The percentage of correct permutations suffers the most when $\alphaInit{}$ is low and $\myPermLR{}$ is high. We can also notice that as we increase $\myPermLR{}$ we also need to increase $\alphaInit{}$ to maintain a good permutation accuracy. 





\section{Conclusion}
We proposed PermLL, a loss correction framework for learning from noisy labels. PermLL generalizes the idea of learning the labels directly, proposed in previous methods, by using instance-dependent permutation layers that can be applied to either the labels or the predictions. We theoretically investigated the differences between the two approaches and demonstrated that applying the permutation layer to the predictions has advantageous properties. Additionally, we experimentally validated the effectiveness of PermLL on synthetic and real datasets by achieving state-of-the-art performance. Finally, we analyzed the effect of adjusting the hyperparameters of PermLL on the percentage of correctly learned permutation layers.

\bibliography{main.bib}

\end{document}